\documentclass[letterpaper, 10 pt, conference]{ieeeconf}  

\IEEEoverridecommandlockouts                              

\usepackage{graphics} 
\usepackage{epsfig} 
\usepackage{mathptmx} 
\usepackage{times} 
\usepackage{amsmath} 
\usepackage{amssymb}  
\usepackage{calc}

\usepackage[shortcuts]{extdash}

\usepackage{graphicx}

\usepackage{booktabs}

\usepackage{verbatim}

\usepackage{amsmath}
\usepackage{psfrag}
\usepackage{subfigure}
\usepackage[cmex10]{mathtools}             
\usepackage{amsfonts,amssymb}
\usepackage{mathrsfs}                      
\usepackage{url}
\usepackage{doi}
\title{\LARGE \bf
The Effect of Communication Topology on Scalar Field Estimation by Networked Robotic Swarms
}

\author{Ragesh K. Ramachandran$^{1}$ and Spring Berman$^{1}$
\thanks{This work was supported by National Science Foundation (NSF) award no. CMMI-1363499.}
\thanks{$^{1}$Ragesh K. Ramachandran and Spring Berman are with the School for Engineering of Matter, Transport and Energy, Arizona State University, Tempe, AZ 85287, USA  {\tt\small rageshkr@asu.edu, spring.berman@asu.edu}}%
}

\begin{document}

\maketitle
\thispagestyle{empty}
\pagestyle{empty}

\begin{abstract}
This paper studies the problem of reconstructing a two-dimensional scalar field using a swarm of networked robots with local communication capabilities. We consider the communication network of the robots to form either a chain or a grid topology.  We formulate the reconstruction problem as  an optimization problem that is constrained by first-order linear dynamics on a large, interconnected system.  To solve this problem, we employ an optimization-based scheme that uses a gradient-based method with an analytical computation of the gradient. In addition, we derive bounds on the trace of observability Gramian of the system, which helps us to quantify and compare the estimation capability of chain and grid networks. A comparison based on a performance measure related to the $\mathcal{H}_2 $ norm of the system is also used to study robustness of the network topologies. Our resultsare validated using both simulated scalar fields and actual ocean salinity data.  
\end{abstract}



\section{INTRODUCTION}
\label{sec:intro}

\PARstart{S}{warm} robotics has emerged as important area of research in recent years. Robotic swarms can be employed in scenarios where an individual robot is incapable of or inefficient at performing a task.  Networked robotic swarms can be used for distributed sensing and estimation in various applications such as environmental monitoring, field surveillance, multi-target tracking, and geo-scientific exploration \cite{iyengar2012distributed} for large area monitoring. In numerous multi-robot applications like formation control\cite{desai2001modeling}, control of mobile platoons \cite{antonelli2006kinematic}, multi target tracking \cite{ahmad2013cooperative}, and sensor networks for field reconstruction \cite{pequito2013optimal} interesting problems related to network topology and configuration arise. One of the primary motivations behind the work done in this paper is to quantify the fundamental limitations that can emerge in these applications due to the chosen topology of the network structure. In particular,  chain and grid topologies are common network choices for  multi-robot applications. We find that even for first-order information dynamics, the topology of the network affects its performance on estimation and robustness to noise. This implies that topology of network could practically affect the performance of algorithms for large inter-connected multi-robot systems.




In this paper, we present a method to estimate the full set of initial measurements of a static scalar field that are obtained by a networked robotic swarm using the temporal data collected by a subset of accessible robots. The robots communicate with their neighbors through a fixed communication network topology. A first-order linear dynamical model is used to describe the information flow in the network. This procedure can also be adapted to estimate a time-varying scalar field whose dynamics are slower than the network information dynamics.  From a control theory perspective, the problem is essentially to find the initial condition of a linear dynamical system given its inputs and outputs. The solution to this problem is associated with observability of the system. 


Although there is a great deal of literature on optimal control, little work has addressed the optimal estimation of initial conditions other than through the inversion of the observability Gramian \cite{kasac2012initial}. In general, the observability of a linear dynamical system can be verified by using the Kalman rank condition \cite{Hespanha2009}.  However, checking the rank condition for large interconnected systems is computationally intensive due to the high dimensionality of the observability Gramian. For this reason, a less computationally intensive graph-theoretic characterization of observability has been more widely used than a matrix-theoretic characterization for large complex networked systems. The observability of complex networks is studied in \cite{Liu12022013} using a graph-based approach, which presents a general result that holds true for most of the chosen network parameters (the edge weights).  In \cite{Meng_Egerstedt4434659}, a graph-theoretic  approach based on equitable partitions graphs is used to derive necessary conditions for observability of networks. Alternately, \cite{Yuan2013} use a matrix-theoretic approach to develop a maximum multiplicity theory to characterize the exact controllability of a network in terms of the minimum number of required independent controller nodes based on the network spectrum. 

In our problem, we focus only on the grid and chain network topologies. The main reason for this choice is the fact that these networks are common candidates for approximating 1D and 2D domains in practical applications. Another major reason behind restricting our focus on the analysis of these networks is because the necessary and sufficient conditions for the observability spectral properties of these networks are well understood from literature \cite{Notarstefano2013,Edwards2013}. In terms of the analysis done in this paper we adopt a quantitative measure of observability based on the trace of observability Gramian similar to \cite{PasqualettiZB13,Pasqualetti20147039448,Yan2015,EnyiohaRPJ14,Muller:1972:AOC:2244776.2244812}, departing from graph theoretic methods used in \cite{Meng_Egerstedt4434659,Liu12022013,Notarstefano2013,Parlangeli2013}. 


The main contributions of this paper are as follows. First, we propose a method to estimate the initial condition of linear dynamical network system with large dimensions using an optimization framework and deriving the gradient required to solve it. Second, we derive bounds on the trace of the observability Gramian of the network system and use these results to compare the estimation capability of grid and chain networks. Third, we use a performance measure based on the $\mathcal{H}_2 $ norm of a system to quantify the robustness of these network system to noise.  We illustrate our approach on both simulated and actual two-dimensional scalar fields.

The  paper  is organized  as  follows. \ref{sec:math_prelim} introduces mathematical concepts and terminology that are used in the paper. \ref{sec:prob_def} describes the problem statement and outlines the assumptions made in its formulation. The network model is presented in \ref{sec:modeling}. \ref{sec:recons} delineates how the scalar field reconstruction can be posed as an optimization problem and computes the analytical gradient required for its solution. Simulation details and results are described in \ref{sec:simulation}.  We derive bounds on the trace of the observability Gramian in \ref{sec:comp_net_toplogy}, which aids us in comparing network topologies. \ref{sec:Perform analyse} discusses a performance analysis of the network topologies based on the $\mathcal{H}_2 $ norm of the system. Finally, \ref{sec:conclusion} concludes the paper and proposes future work.

\section{MATHEMATICAL PRELIMINARIES}
\label{sec:math_prelim}


A graph $\mathcal{G}$ can be defined as the tuple $\left( \mathit{V(\mathcal{G})}, \mathit{E(\mathcal{G})} \right)$, where $\mathit{V(\mathcal{G})}$ is a set of $N$ vertices, or {\it nodes}, and $\mathit{E(\mathcal{G})} = \left\lbrace (i,j) : i \neq j,\ i,j \in \mathit{V(\mathcal{G})} \right\rbrace$ is a set of $M$ {\it edges}.  Nodes $i$ and $j$ are called {\it neighbors}  if $(i,j) \in \mathit{E(\mathcal{G})}$. The set of neighbors of node $i$ is denoted by $ \mathcal{N}_i = \left\lbrace j : j \in \mathit{V(\mathcal{G})}, (i,j) \in \mathit{E(\mathcal{G})} \right\rbrace$.  The {\it degree} $d_i$ of a node $i$ is defined as $\left| \mathcal{N}_i \right|$.  We assume that $\mathcal{G}$ is finite, simple, and connected unless mentioned otherwise.  


A graph $\mathcal{G}$ is associated with several matrices whose spectral properties will be used to derive our results.  The \textit{incidence matrix} of a graph with arbitrary orientation is defined as $\mathbf{B}(\mathcal{G}) = [b_{ij}] \in \mathbb{R}^{N \times M}$, where the entry $b_{ij} = 1$ if $i$ is the initial node of some edge $j$ of $ \mathcal{G}$, $b_{ij} = -1$ if $i$ is the terminal node of some edge $j$ of $ \mathcal{G}$, and $b_{ij} = 0$ otherwise.  It can be shown that the left nullspace of $\mathbf{B}(\mathcal{G})$ is  $c\mathbf{1}_N,\ c \in \mathbb{R}$, where $\mathbf{1}_N$ is the $N \times 1$ vector of ones \cite{GodsilRoyle2001}.  The \textit{degree matrix} $\Delta(\mathcal{G})$ of a graph is given by $\Delta(\mathcal{G}) = Diag(d_1, ..., d_N)$. The {\it adjacency matrix} $\mathbf{A}(\mathcal{G}) = [a_{ij}] \in \mathbb{R}^{N \times N}$ has entries $a_{ij} = 1$ when $(i,j) \in \mathit{E(\mathcal{G})}$ and $a_{ij} = 0$ otherwise.  The  \textit{graph Laplacian} can be defined from these two matrices as $\mathbf{L}(\mathcal{G}) = \Delta(\mathcal{G}) - \mathbf{A}(\mathcal{G})$.  The Laplacian of an undirected graph is symmetric and positive semidefinite, which implies that it has real nonnegative eigenvalues $\lambda_i(\mathcal{G})$, $i = 1,...n$.  The eigenvalues can be ordered as $\lambda_1(\mathcal{G}) \leq \lambda_2(\mathcal{G}) \leq ... \leq \lambda_N(\mathcal{G})$, where $\lambda_1(\mathcal{G}) = 0$. The eigenvector corresponding to eigenvalue $\lambda_1(\mathcal{G})$ can be computed to be $\mathbf{1}_N $. By Theorem 2.8 of \cite{Egerstedt2010}, the graph is connected if and only if $\lambda_2(\mathcal{G}) > 0$. 

Several other matrices will be defined as follows.  An $n_1 \times n_2$ identity matrix will be denoted by $\mathbf{I}_{n_1 \times n_2}$, and an $n_1 \times n_2$ matrix of zeros will be denoted by $\mathbf{0}_{n_1 \times n_2}$.  The matrix $\mathbf{J}_N$ is defined as $\mathbf{J}_N = \mathbf{1}_N^T \mathbf{1}_N$. 



\begin{figure}[t]
	\centering
	\subfigure[Chain topology  \label{fig:Chain Topology} ]{\includegraphics[width=.8\linewidth,height = 4cm ]{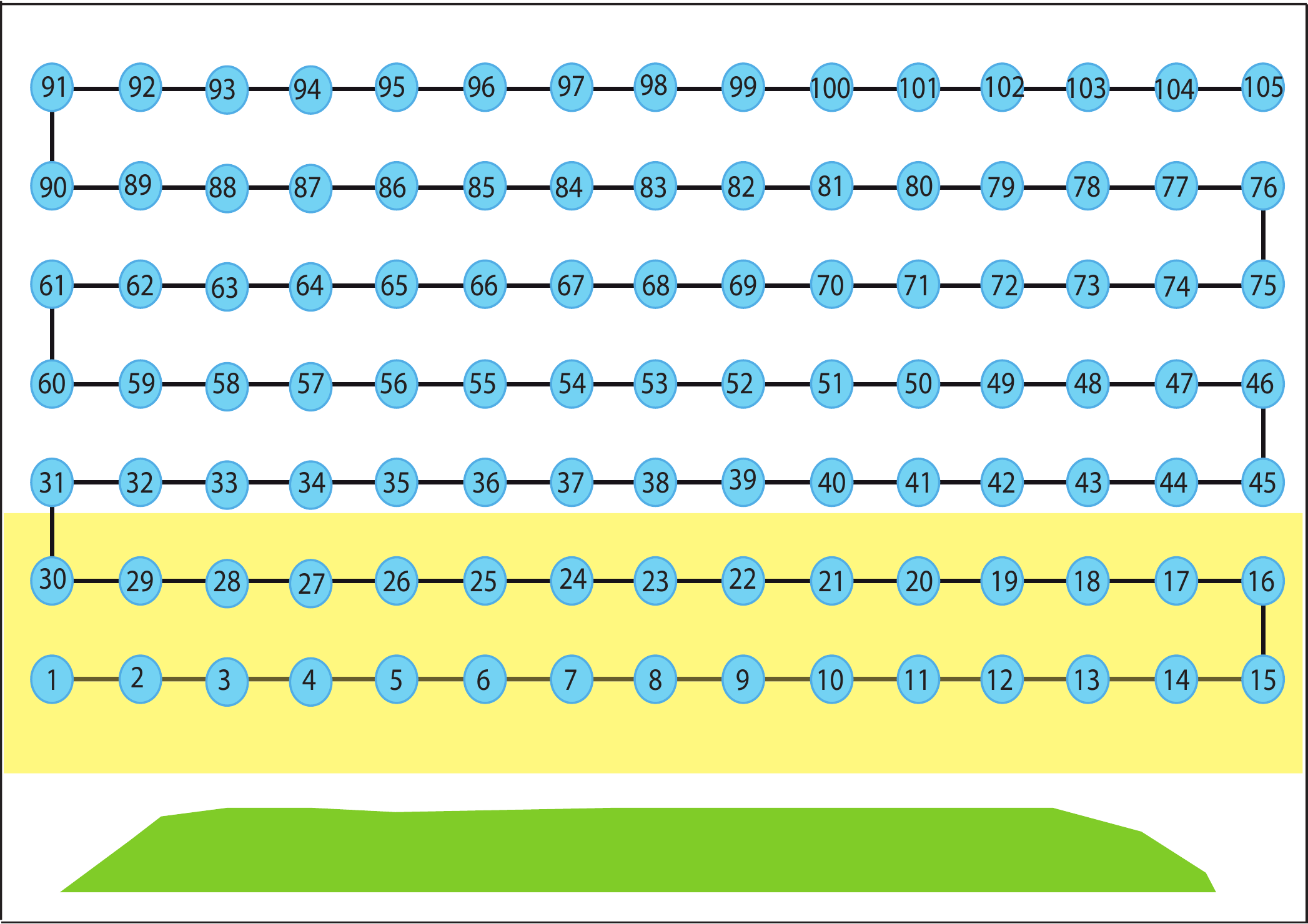}}	
	\subfigure[Grid topology \label{fig:Grid Topology}]{\includegraphics[width=.8\linewidth,height = 4cm ]{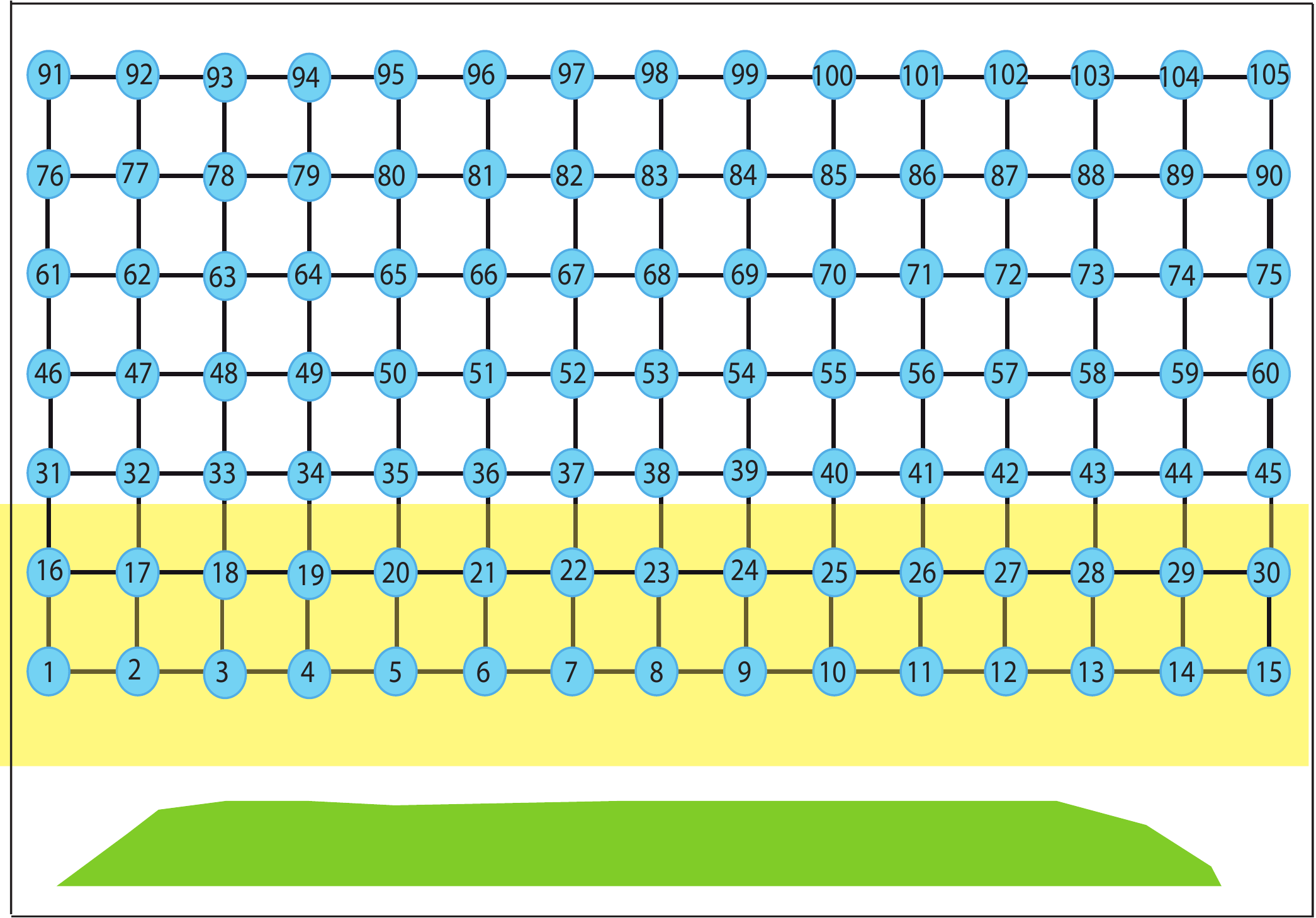}} 
	\caption{Illustration of the chain and grid network topologies. The blue circles are robots and are labeled by numbers. Robots in the yellow region are accessible robots.}
	\label{fig:topologies}
\end{figure}


%
%
%
%
%
%
%


\section{PROBLEM STATEMENT}
\label{sec:prob_def}

Consider a set of $N$ robots with local communication ranges and local sensing capabilities. The robots are arranged in a bounded domain as shown in  \ref{fig:topologies}. Each robot is capable of measuring the value of a scalar field at its location and communicating this value to its neighbors, which are defined as the robots that are within its communication range. The robots take measurements at some initial time and transmit this information using a nearest-neighbor averaging rule, which is described in \ref{sec:modeling}. As shown in \ref{fig:topologies}, we assume to have direct access only to the measurements of a small subset of the robots, which we call the {\it accessible robots}, which for instance may be closer to a particular boundary of the domain. We also assume that the robot positions are predetermined and that the robots employ feedback mechanisms to regulate their positions in the presence of external disturbances.


We address the problem of reconstructing the initial measurements taken by all the robots from the measurements of the accessible robots. This can be formulated as the problem of determining whether the information flow dynamics in the network are observable with respect to a set of given outputs. As mentioned in \ref{sec:intro}, this is a difficult problem to solve for arbitrary communication topologies.  Hence, we restrict our investigation to chain and grid communication topologies, whose structural observability properties are well-studied \cite{Notarstefano2013,Parlangeli2013}. We will focus on comparing the chain and grid topologies in terms of their utility as communication networks to be used in reconstructing an initial set of data. 


\section{NETWORK MODEL}
\label{sec:modeling}

The communication network among the $N$ robots is represented by an undirected graph $ \mathcal{G} = \left( \mathit{V(\mathcal{G})}, \mathit{E(\mathcal{G})} \right)$, where vertex $i \in \mathit{V(\mathcal{G})}$ denotes robot $i$, and robots $i$ and $j$ can communicate with each other if $(i, j) \in \mathit{E(\mathcal{G})}$.   Let $x_i(t) \in \mathbb{R}$ be a scalar data value obtained by robot $i$ at time $t$. We define the information flow dynamics of robot $i$ as
\begin{equation}
\frac{dx_i}{dt} = \sum_{(i, j) \in \mathcal{N}_i } (x_j-x_i).  
\label{eqn:agent_dyn}
\end{equation}
  

The vector of all robots' information at time $t$ is denoted by $ \mathbf{X}(t) = [x_1(t) ~x_2(t) ~... ~x_N(t)]^T$.  Using \ref{eqn:agent_dyn} to define the dynamics of $x_i(t)$ for each robot $i$, we can define the information flow dynamics over the entire network as
\begin{eqnarray}
\dot{\mathbf{X}}(t) &=& - \mathbf{L}(\mathcal{G}) \mathbf{X}(t), \nonumber\\ 
\mathbf{X}(0) &=& \mathbf{X}_0,
\label{eqn:sys_dyn}
\end{eqnarray}
where $\mathbf{X}_0 \in \mathbb{R}^N$ contains the unknown initial values of the data obtained by the robots at time $t=0$, which is the information that we want to estimate. 


We define $\mathit{Id} = \left\lbrace I_1, I_2, ..., I_k \right\rbrace  \subseteq \mathit{V(\mathcal{G})}$ as the index set of the {\it accessible robots}.  The output equation for the linear system \ref{eqn:sys_dyn} is given by
\begin{equation}
\mathbf{Y}(t) =  \mathbf{C} \mathbf{X}(t), \\ 
\label{eqn:sys_out}
\end{equation}
where $\mathbf{Y}(t) \in \mathbb{R}^k$ and $\mathbf{C} =[c_{ij}] \in \mathbb{R}^{k \times N}$ is a sparse matrix whose entries are defined as  $c_{ij} = 1$ if $ i=j$ and $i \in \mathit{Id} $, $c_{ij} = 0$ otherwise.  If we number the robots in such a way that the first $k$ output nodes (robots) are ordered from $1$ to $k$, then $\mathbf{C} = \left[ \mathbf{I}_{k \times k}\ \mathbf{0}_{k \times (N - k)} \right]$.


 As previously discussed, we focus on the case where the network has a chain or grid communication topology. The type of topology affects the network dynamics through its associated \textit{graph Laplacian} $\mathbf{L}(\mathcal{G}_g)$. Let $ \mathcal{G}_c$ and $ \mathcal{G}_g $ represent communication networks with a chain topology and a grid topology, respectively. When the robots in each network are labeled as shown in \ref{fig:Chain Topology} and \ref{fig:Grid Topology}, then it can be shown that $\mathbf{L}(\mathcal{G}_c)$ and $ \mathbf{L}(\mathcal{G}_g) $ \cite{Edwards2013} have the following structures: 
\begin{equation}
 \mathbf{L}(\mathcal{G}_c) = 
 \begin{bmatrix}
 1  & -1  & 0 & \cdots & \cdots & \cdots & \cdots & 0 \\
 -1  & 2  & -1  & \ddots & && & \vdots \\
 0 & -1  & 2 & -1  & \ddots & &  & \vdots \\
 \vdots & \ddots & \ddots & \ddots & \ddots & \ddots &  & \vdots \\
 \vdots & & \ddots & \ddots & \ddots & \ddots & \ddots& \vdots\\
 \vdots  & & & \ddots & -1  & 2  &  -1  & 0\\
 \vdots  & && & \ddots & -1  & 2  &  -1\\
 0 & \cdots &  \cdots & \cdots & \cdots & 0 & -1 & 1  \\
 \end{bmatrix} 
\label{eqn:Lap_chain}
\end{equation}  
and
\begin{equation}
\mathbf{L}(\mathcal{G}_g) = 
\begin{bmatrix}
\mathbf{D}_1  & -\mathbf{I}  & 0 & \cdots & \cdots & \cdots & \cdots & 0 \\
-\mathbf{I}  & \mathbf{D}_2  & -\mathbf{I}  & \ddots & && & \vdots \\
0 & -\mathbf{I}  & \mathbf{D}_2 & -\mathbf{I}  & \ddots & &  & \vdots \\
\vdots & \ddots & \ddots & \ddots & \ddots & \ddots &  & \vdots \\
\vdots & & \ddots & \ddots & \ddots & \ddots & \ddots& \vdots\\
\vdots  & & & \ddots & -\mathbf{I}  & \mathbf{D}_2  &  -\mathbf{I}  & 0\\
\vdots  & && & \ddots & -\mathbf{I}  & \mathbf{D}_2 &  -\mathbf{I}\\
0 & \cdots &  \cdots & \cdots & \cdots & 0 & -\mathbf{I} & \mathbf{D}_1  \\
\end{bmatrix},
\label{eqn:Lap_grid}
\end{equation}
where
\begin{align}
\mathbf{D}_1 = 
\begin{bmatrix}
2  & -1     &   \cdots     &   \cdots     &     0       \\
-1 & 3      & -1     &        &     \vdots        \\
\vdots &\ddots  & \ddots & \ddots &      \vdots       \\
\vdots &       &   -1   &  3     &  -1      \\
0  &     \cdots   &  \cdots      &     -1 &   2       \\
\end{bmatrix}, \nonumber\\
\mathbf{D}_2 = 
\begin{bmatrix}
3  & -1     &   \cdots      &     \cdots    &0\\
-1 & 4      & -1     &        & \vdots \\
\vdots &\ddots  & \ddots & \ddots    & \vdots  \\
\vdots &        &   -1   &  4        &-1\\
0  &     \cdots    &    \cdots     &     -1 &3\\
\end{bmatrix}. \nonumber
\end{align}
Here, $ \mathbf{L}(\mathcal{G}_g) $ is a $(l_1 l_2) \times (l_1 l_2)$ matrix and $\mathbf{D}_1, \mathbf{D}_2$ are both $l_1 \times l_1$ matrices, with $l_1 l_2 = N$. Without loss of generality, we assume that the grid is square, meaning that $l_1 = l_2 = l$. We direct the reader to \cite{Banham2006} for a numerical example of  $\mathit{L(\mathcal{G}_g)}$. 


The graph Laplacians $\mathbf{L}(\mathcal{G}_c)$ and $\mathbf{L}(\mathcal{G}_g)$ are constructed based on the numbering of the vertex sets $ \mathit{V(\mathcal{G}_c)} $ and $ \mathit{V(\mathcal{G}_g)} $ that is shown in \ref{fig:topologies}. Graphs that are constructed by reordering the vertices of the graphs shown in \ref{fig:topologies}  are isomorphic to the graphs in the figure. Isomorphic graphs are also isospectral \cite{Wilson_Zhu_20082833}. 

Since the system \ref{eqn:sys_dyn} is linear, its solution is \cite{Hespanha2009}
\begin{equation}
\mathbf{X}(t) = \mathbf{e}^{- \mathbf{L}(\mathcal{G})t} \mathbf{X}_0.
\label{eqn:sys_dyn_sol}
\end{equation}

By combining \ref{eqn:sys_out} and \ref{eqn:sys_dyn_sol}, we obtain the map between the unknown initial data $\mathbf{X}_0$ and the measured output $\mathbf{Y}(t)$ as
\begin{equation}
\mathbf{Y}(t) = \mathbf{C}\mathbf{e}^{- \mathbf{L}(\mathcal{G})t} \mathbf{X}_0.
\label{eqn:in_out_map}
\end{equation}



\section{FIELD RECONSTRUCTION}
\label{sec:recons}

The problem of scalar field reconstruction can now be framed as an inversion of the map given by \ref{eqn:in_out_map}. From linear systems theory, the property of {\it observability} refers to the ability to determine an initial state $ \mathbf{X}_0$ from the inputs and outputs of a linear dynamical system \cite{Hespanha2009}.  For systems defined by \ref{eqn:sys_dyn} with an associated chain or grid topology, the conditions for observability are well-studied \cite{Notarstefano2013}. This ensures that the reconstruction problem can be solved for the types of networks that we consider. 

We solve the scalar field reconstruction problem by posing it as an optimization problem. The optimization procedure uses observed data  $ \hat{\mathbf{Y}}(t)$ from the accessible robots over the time interval $t \in [0\ T]$ to recover $ \mathbf{X}_0 $. The goal of the optimization routine is to find the state $\mathbf{X}_0$ that minimizes the normed distance between this observed data, $\hat{\mathbf{Y}}(t)$, and the output $\mathbf{Y}(t) $ computed using \ref{eqn:in_out_map}. Therefore, we can frame our optimization objective as the computation of $ \mathbf{X}_0$ that minimizes the functional $J(\mathbf{X}_0)$, defined as
\begin{equation}
 J(\mathbf{X}_0) = \frac{1}{2}\int_{0}^{T}\left \| \mathbf{Y}(t) - \hat{\mathbf{Y}}(t) \right \|_2^2 dt + \frac{\lambda}{2}\left \| \mathbf{X}_0 \right \|^2,
\label{eqn:obj_fun}
\end{equation}
subject to the constraint given by \ref{eqn:in_out_map}. Here, $\lambda$ is the Tikhonov regularization parameter, which is added to the objective function to prevent $\mathbf{X}_0 $ from becoming large due to noise in the data \cite{Boyd:2004:CO:993483}.


The convexity of  $J(\mathbf{X}_0)$ ensures the convergence of gradient descent methods to its global minima. We use one such method to compute the $ \mathbf{X}_0  $ that minimizes this functional. The method requires us to compute the gradient of  $J(\mathbf{X}_0)$ with respect to $ \mathbf{X}_0 $. This is done by combining \ref{eqn:in_out_map} and \ref{eqn:obj_fun}, then taking the G\^{a}teaux derivative of the resulting expression with respect to $ \mathbf{X}_0  $ \cite{Luenberger:1997:OVS:524037}. Defining $\Psi(t) = \mathbf{C}\mathbf{e}^{- \mathbf{L}(\mathcal{G})t}$, the gradient of  $J(\mathbf{X}_0)$ can be computed in this way as:
\begin{equation}
\delta J(\mathbf{X}_0) =\int_{0}^{T} \left( \Psi(t) \right )^* \left( \Psi(t) \mathbf{X}_0 - \hat{\mathbf{Y}}(t) \right ) dt + \lambda \mathbf{X}_0, 
\label{eqn:pre gradient}
\end{equation}
where $ \left(  \Psi(t) \right )^* $ is the Hermitian adjoint of $\Psi(t) $, which in this case is simply the Hermitian transpose \cite{Luenberger:1997:OVS:524037}. 


The most computationally intensive part of calculating \ref{eqn:pre gradient} is computing the matrix exponential in $\Psi(t)$. There has been a great deal of literature about approximate computation of the matrix exponential \cite{Hochbruck1997doi:10.1137/S0036142995280572,Orecchia:2012:AEL:2213977.2214080} which by definition is an infinite matrix series. In general, finding the matrix exponential is a computationally hard problem for very large matrices and computing them can be error-prone if not done carefully, especially if spectral decomposition \cite{Strang88} of the matrix is not possible \cite{Moler2003doi:10.1137/S00361445024180}. We can calculate the gradient by noting that $\mathbf{Y}(t) = \Psi(t)\mathbf{X}_0$ by \ref{eqn:in_out_map}, applying a change of variables $\tau = T - t$ to the integral term in \ref{eqn:pre gradient}, and defining $\hat{u}(\tau) \equiv \mathbf{Y}(T - \tau) - \hat{\mathbf{Y}}(T - \tau)$:
\begin{eqnarray}
\int_{0}^{T} \left( \Psi(t) \right )^* \left( \Psi(t) \mathbf{X}_0 - \hat{\mathbf{Y}}(t) \right ) dt && \nonumber \\
 && \hspace{-5cm} = ~\int_{0}^{T} \left( \Psi(T-\tau) \right )^* \left( \mathbf{Y}(T- \tau) - \hat{\mathbf{Y}}(T - \tau) \right ) d\tau \nonumber \\
   && \hspace{-5cm} = ~\int_{0}^{T}  \mathbf{e}^{- \mathbf{L}^*(\mathcal{G})(T - \tau)}  \mathbf{C}^*   \hat{u}(\tau)  d\tau.   \nonumber
\end{eqnarray}


This expression can be thought of as the solution $P(\tau)$ of the following differential equation at time $\tau=T$ \cite{Hespanha2009}: 
\begin{equation}
\frac{dP}{d\tau} =  - \mathbf{L}^*(\mathcal{G})P(\tau) + \mathbf{C}^*\hat{u}(\tau), ~~~P(0) = 0.
\label{eqn:adjoint}
\end{equation}
Using this result, the gradient \ref{eqn:pre gradient} can be written as 
\begin{equation}
\delta J(\mathbf{X}_0) = P(T) + \lambda \mathbf{X}_0.
\label{eqn:gradient}
\end{equation}
To compute the gradient, we can solve \ref{eqn:adjoint} forward to find $P(T)$.

\section{SIMULATIONS}
\label{sec:simulation}


\begin{figure}[t] 
	\centering
	
	\begin{tabular}{ccc}
		\centering	
		\hspace{-3mm}
		
		\subfigure[Actual field  \label{fig:Acutal_chain} ]{\includegraphics[width=.3\linewidth,height = 4cm ]{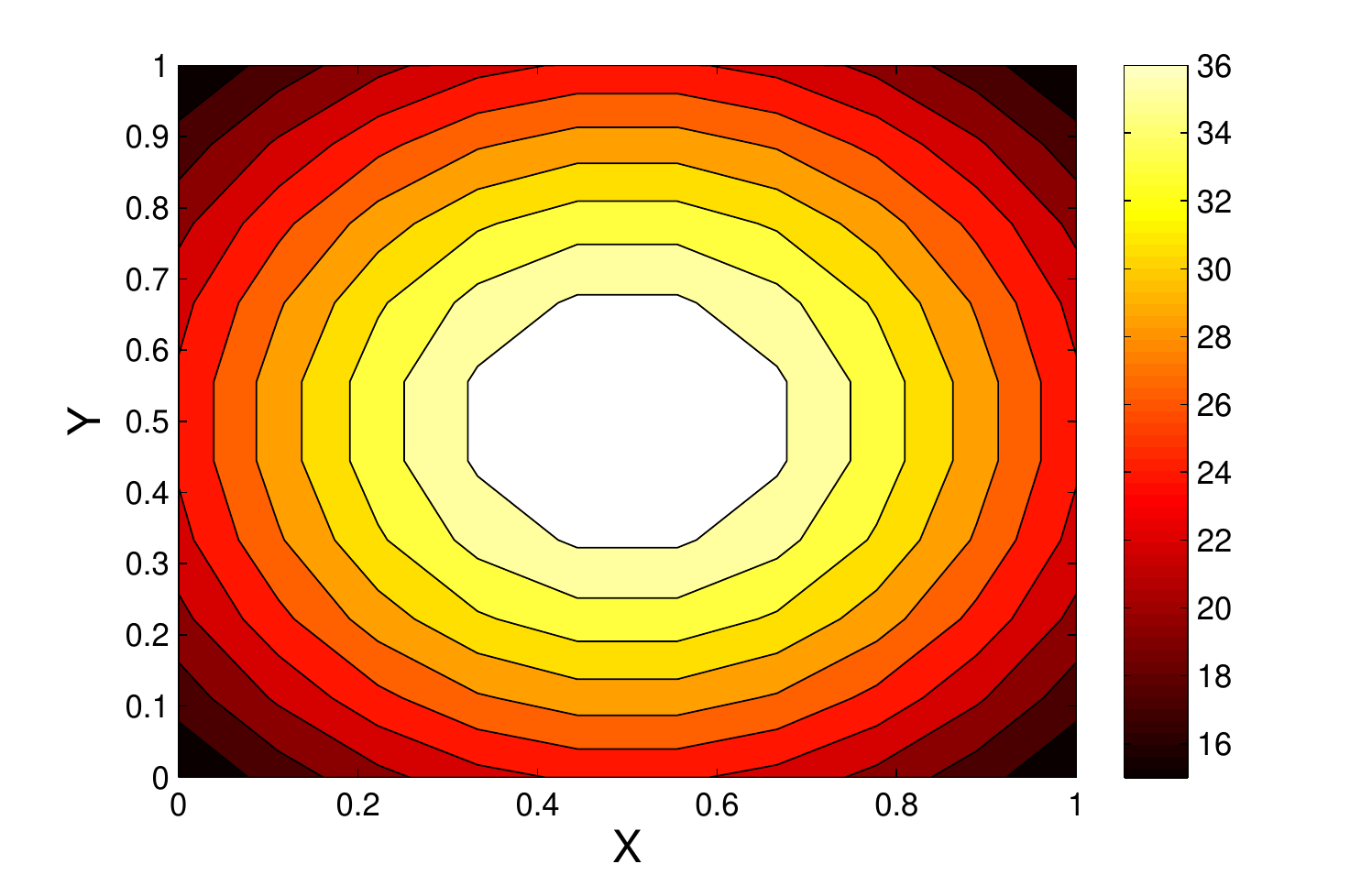}}		
		&
		
		\subfigure[Estimated field \label{fig:estimated chain}]{\includegraphics[width=.3\linewidth,height = 4cm ]{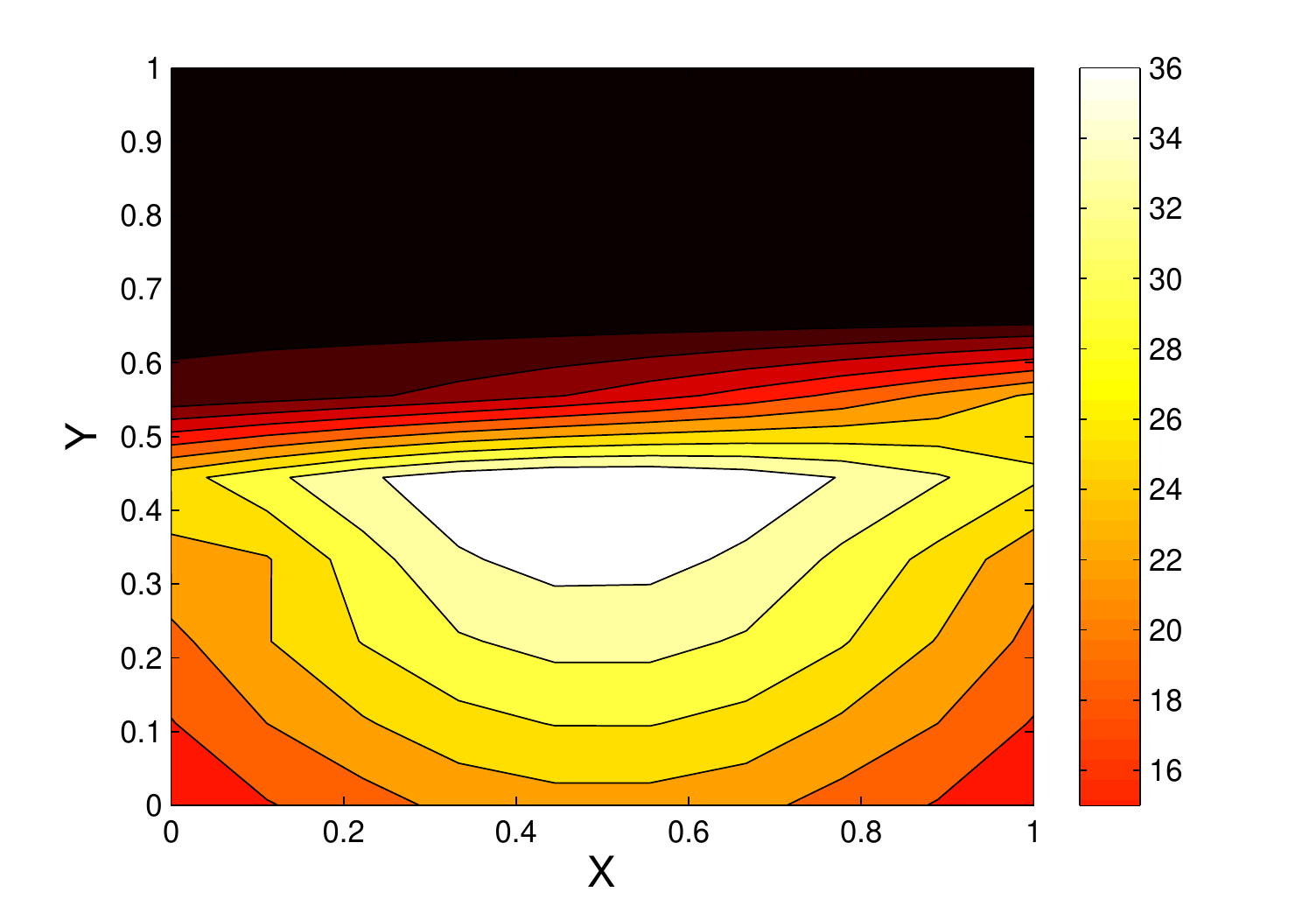}} 
		&
		\subfigure[Absolute value of error \label{fig:error chain}]{\includegraphics[width=.3\linewidth,height = 4cm ]{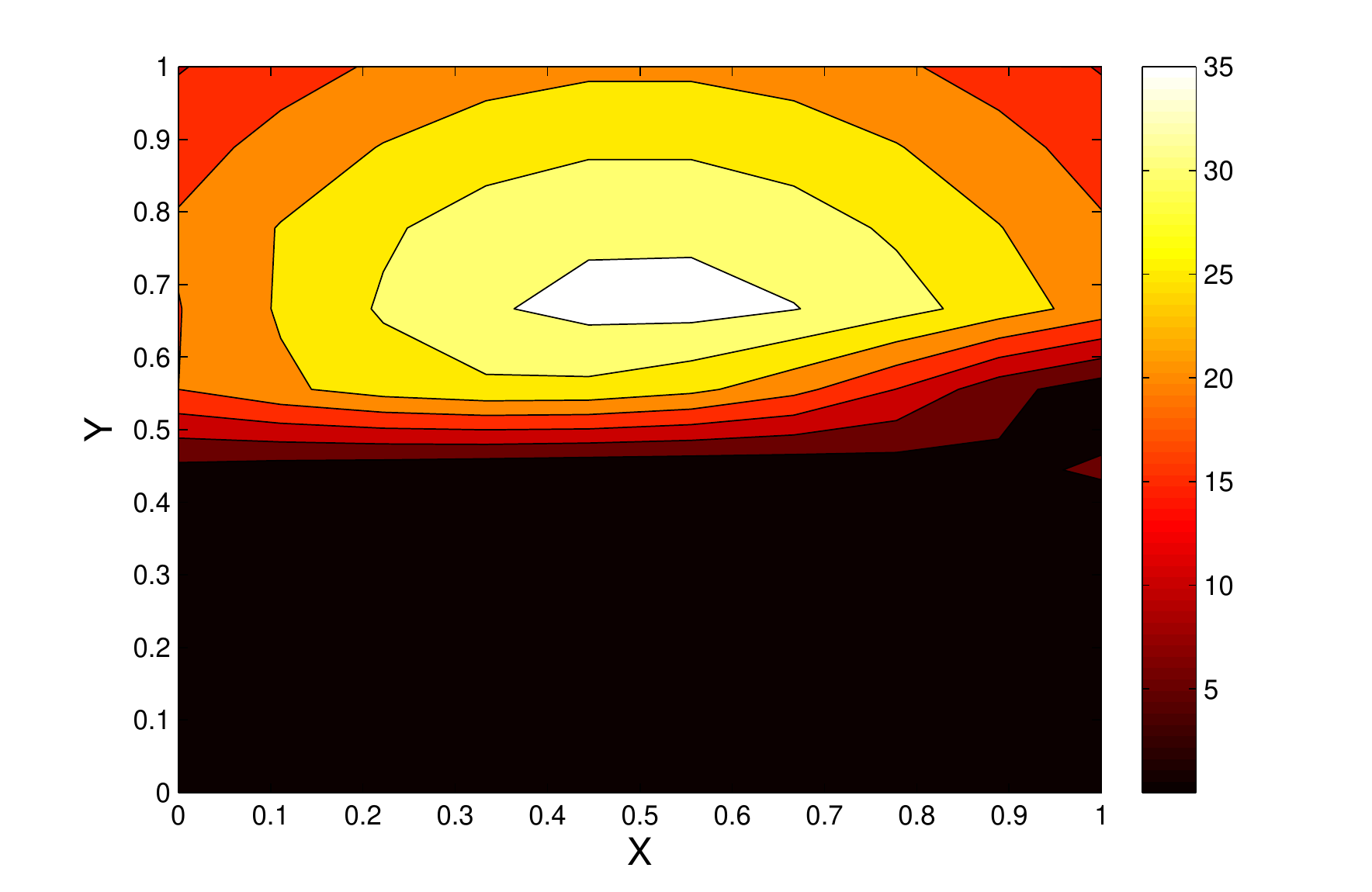}}

	\end{tabular}
	\caption{Gaussian function estimation using 100 robots communicating using a chain topology using temporal data acquired from 30 robots for a time period of 50 seconds }
	\label{fig:estimate chain}
\end{figure}

\begin{figure}[t] 
	\centering
	
	\begin{tabular}{ccc}
		\centering	
		\hspace{-3mm}
		
		\subfigure[Actual field  \label{fig:Acutal_grid} ]{\includegraphics[width=.3\linewidth,height = 4cm ]{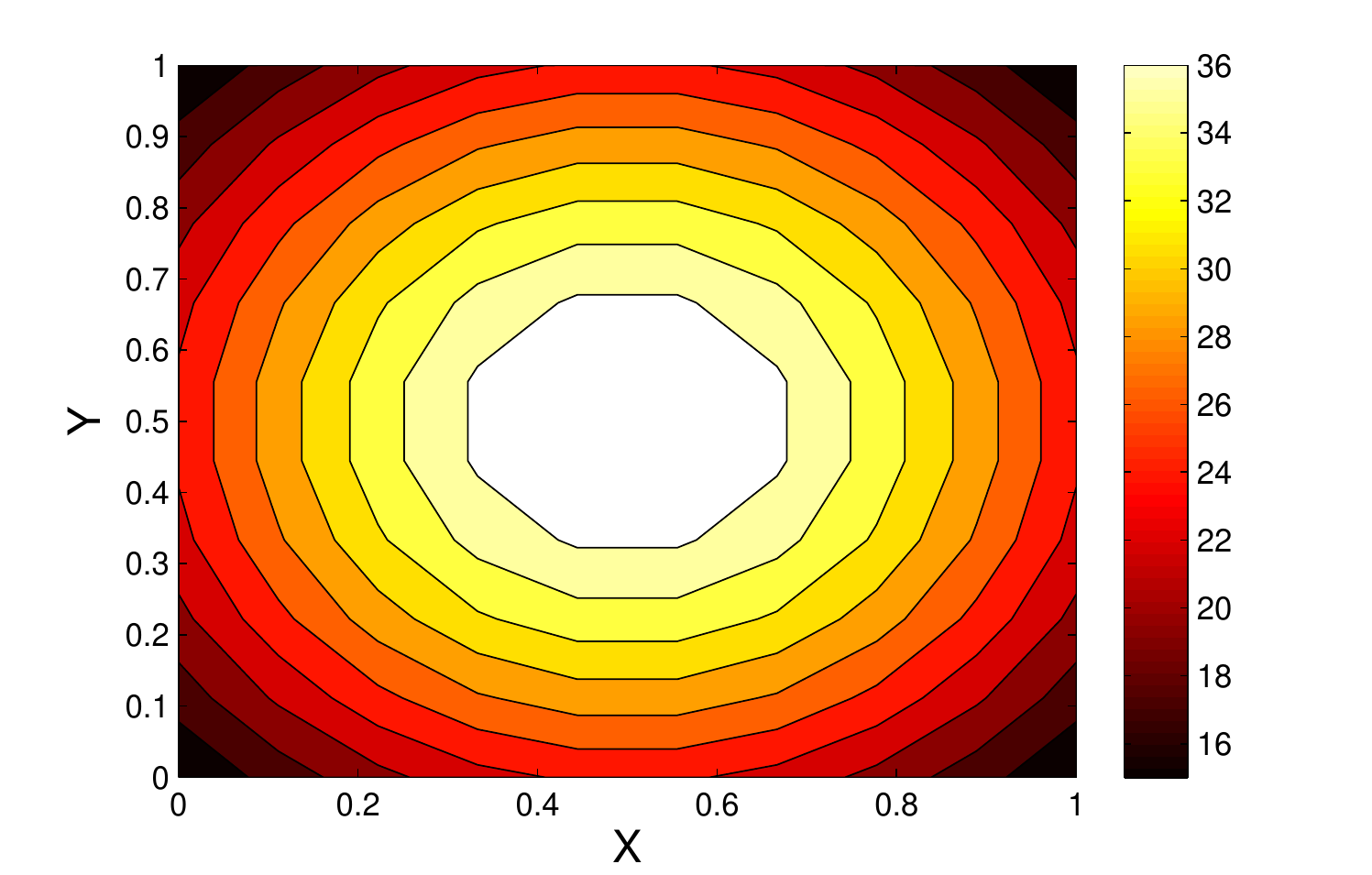}}		
		&
		
		\subfigure[Estimated field \label{fig:estimated grid}]{\includegraphics[width=.3\linewidth,height = 4cm ]{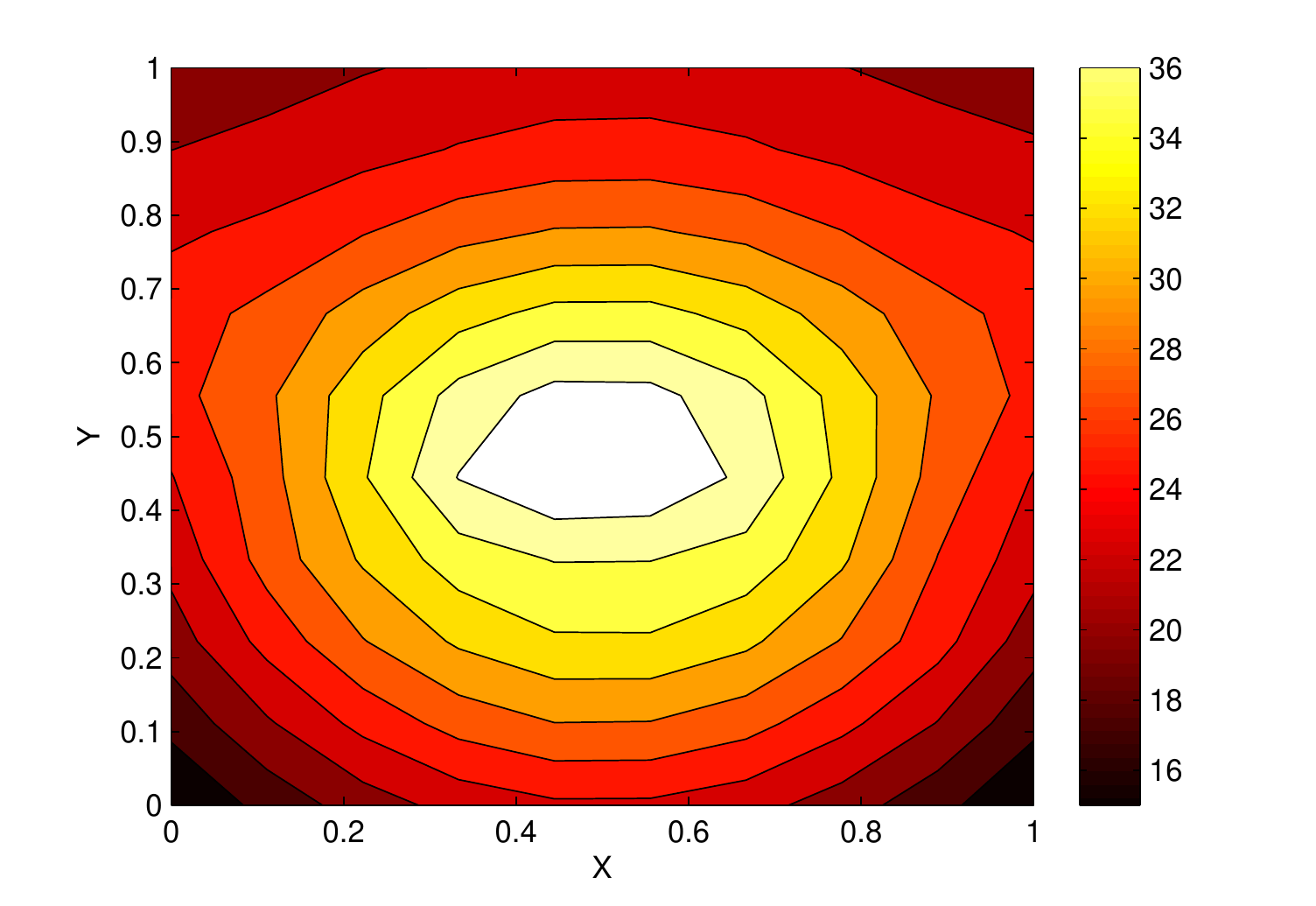}} 
		&
		\subfigure[Absolute value of error \label{fig:error grid}]{\includegraphics[width=.3\linewidth,height = 4cm ]{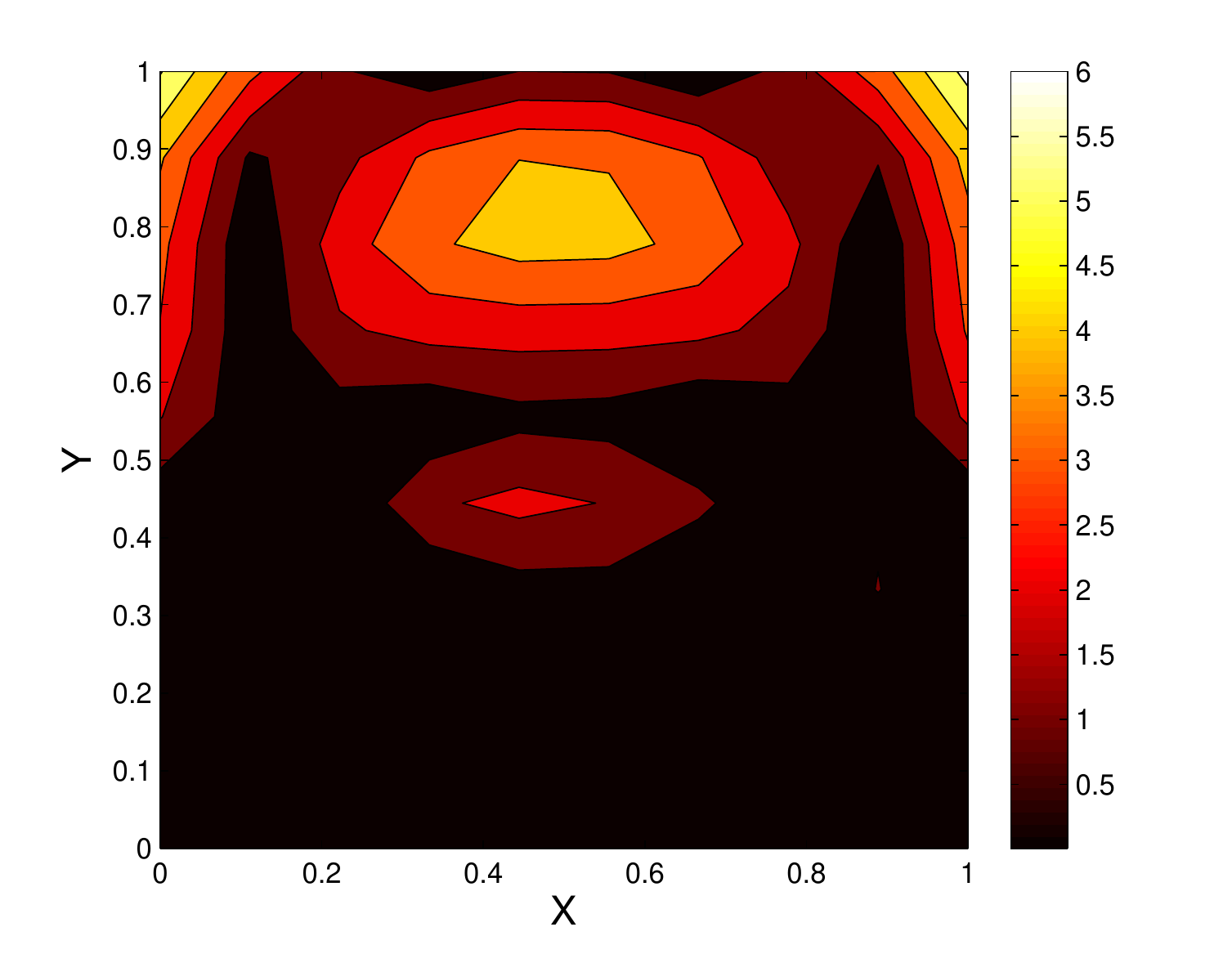}}

	\end{tabular}
	\caption{Gaussian function estimation using 100 robots communicating using a grid topology using temporal data acquired from 30 robots for a time period of 50 seconds }
	\label{fig:estimate grid}
\end{figure}

\begin{figure}[t] 
	\centering
	
	\begin{tabular}{ccc}
		\centering	
		\hspace{-3mm}
		
		\subfigure[Actual field  \label{fig:Acutal_grid_saline} ]{\includegraphics[width=.3\linewidth,height = 3.5cm ]{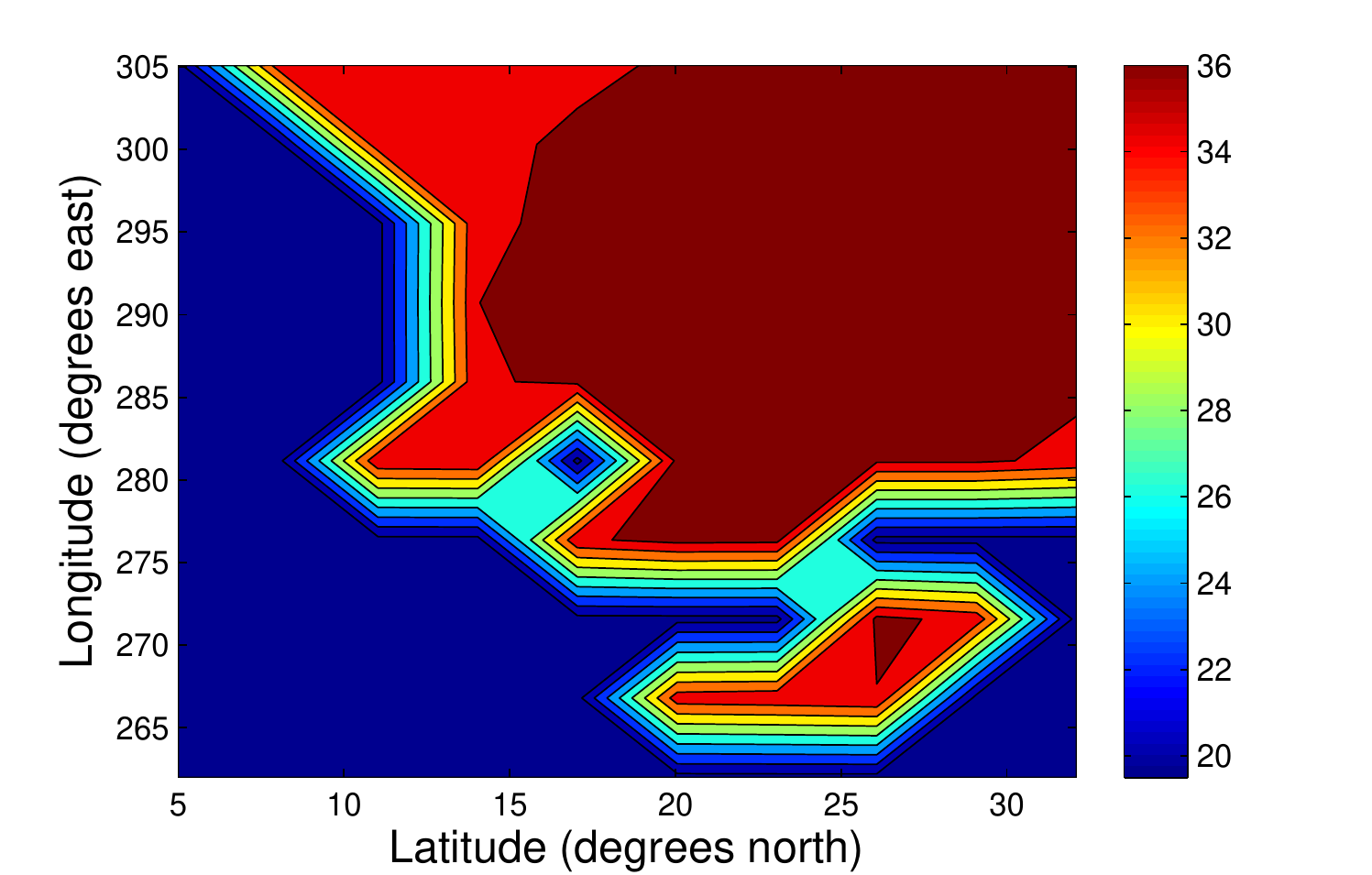}}		
		&
		
		\subfigure[Estimated field \label{fig:estimated grid_saline}]{\includegraphics[width=.3\linewidth,height = 3.5cm ]{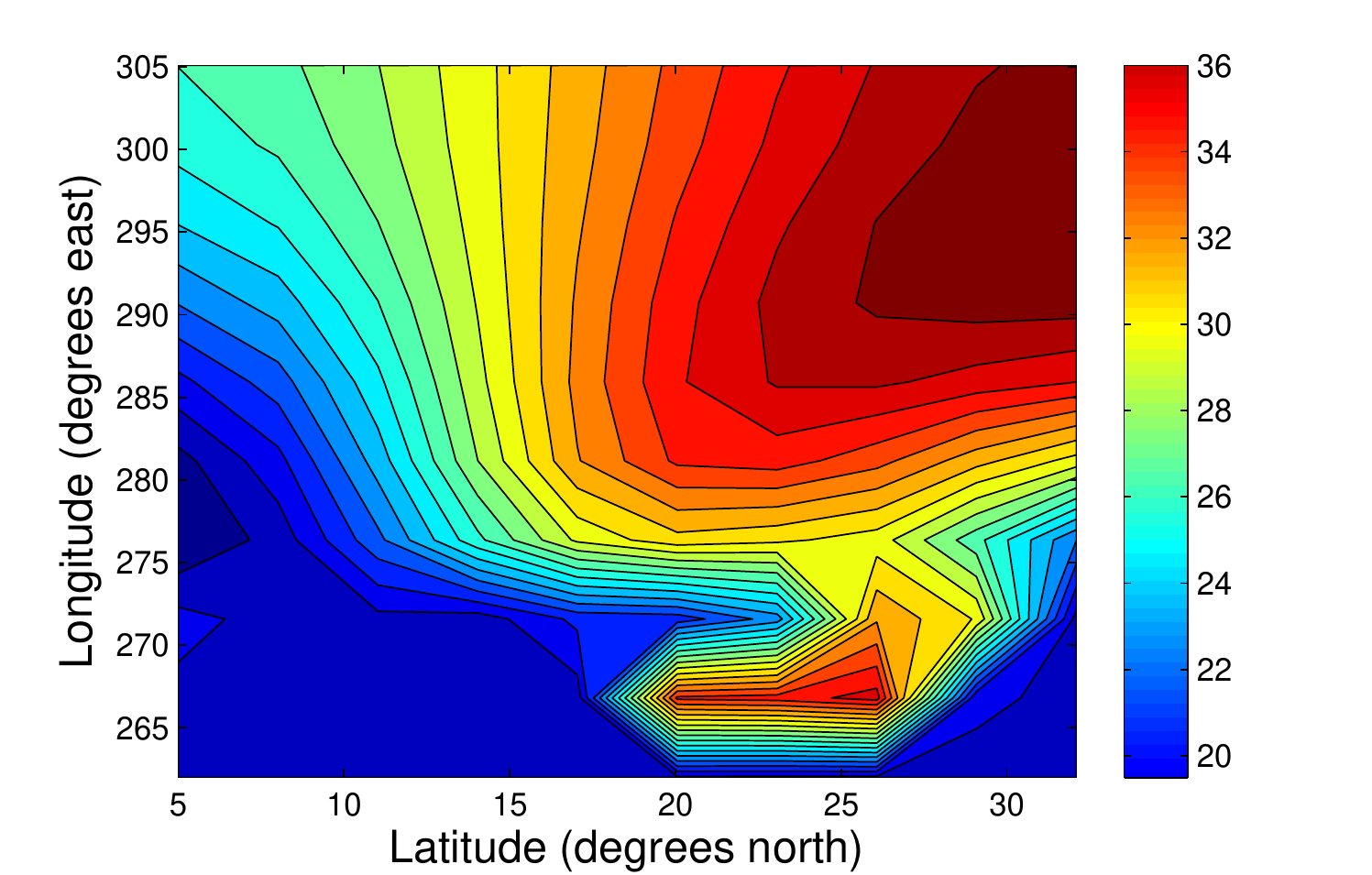}} 
		&
		\subfigure[Absolute value of error \label{fig:error grid_saline}]{\includegraphics[width=.3\linewidth,height = 3.5cm ]{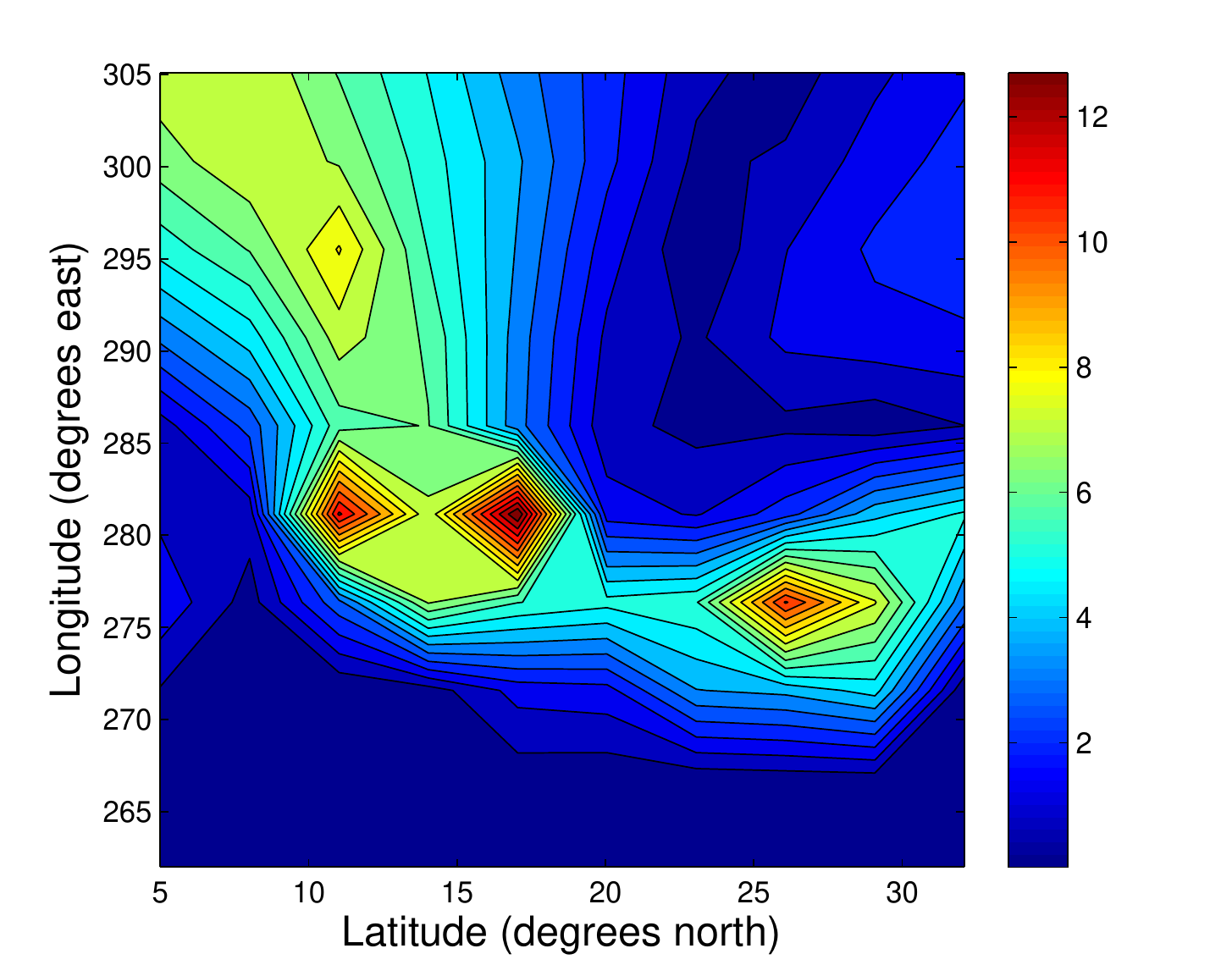}}

	\end{tabular}
	\caption{Estimation based on the data about salinity (psu) of Atlantic ocean at a depth of $25$ meter,  using 100 robots communicating using a grid topology using temporal data acquired from 30 robots for a time period of 50 seconds }
	\label{fig:estimate grid_saline}
\end{figure}


We applied the method described in \ref{sec:recons} to reconstruct a Gaussian scalar field using $100$ robots, whose communication network either has a chain topology or a grid topology.  The simulations were performed on a normalized domain of size 1 m $\times$ 1 m. The field was reconstructed over a time period of 50 sec using the data from a set of 30 accessible robots. \ref{fig:estimate chain} and \ref{fig:estimate grid} illustrate the results from using the chain and grid topologies, respectively. Each figure shows the contour plots of the actual field, the estimated field, and absolute value of the error between these plots.  From these plots, it is evident that the grid topology yields a much more accurate reconstruction of the field than the chain topology, even though both networks can be characterized as observable.  


In order to test the performance of our technique in a practical scenario, we applied it to real data on the salinity (psu) of a section of the Atlantic ocean at a depth of $25$ m, obtained from \cite{nationalcentersforenvironmentinformation2015}.  The field was reconstructed over a time period of 50 sec using $100$ robots with a grid communication topology and 30 accessible robots whose temporal data was sampled at 10 Hz.  The contour plots in \ref{fig:estimate grid_saline} show that the estimated salinity field reproduces the key features of the actual field with reasonable accuracy.





\section{COMPARISON OF NETWORK TOPOLOGIES}  
\label{sec:comp_net_toplogy}


\begin{figure}[t]
	\centering
		\hspace{-3mm}
		\subfigure[Networks with $100$ nodes  \label{fig:bound 100} ]{\includegraphics[scale=0.55]{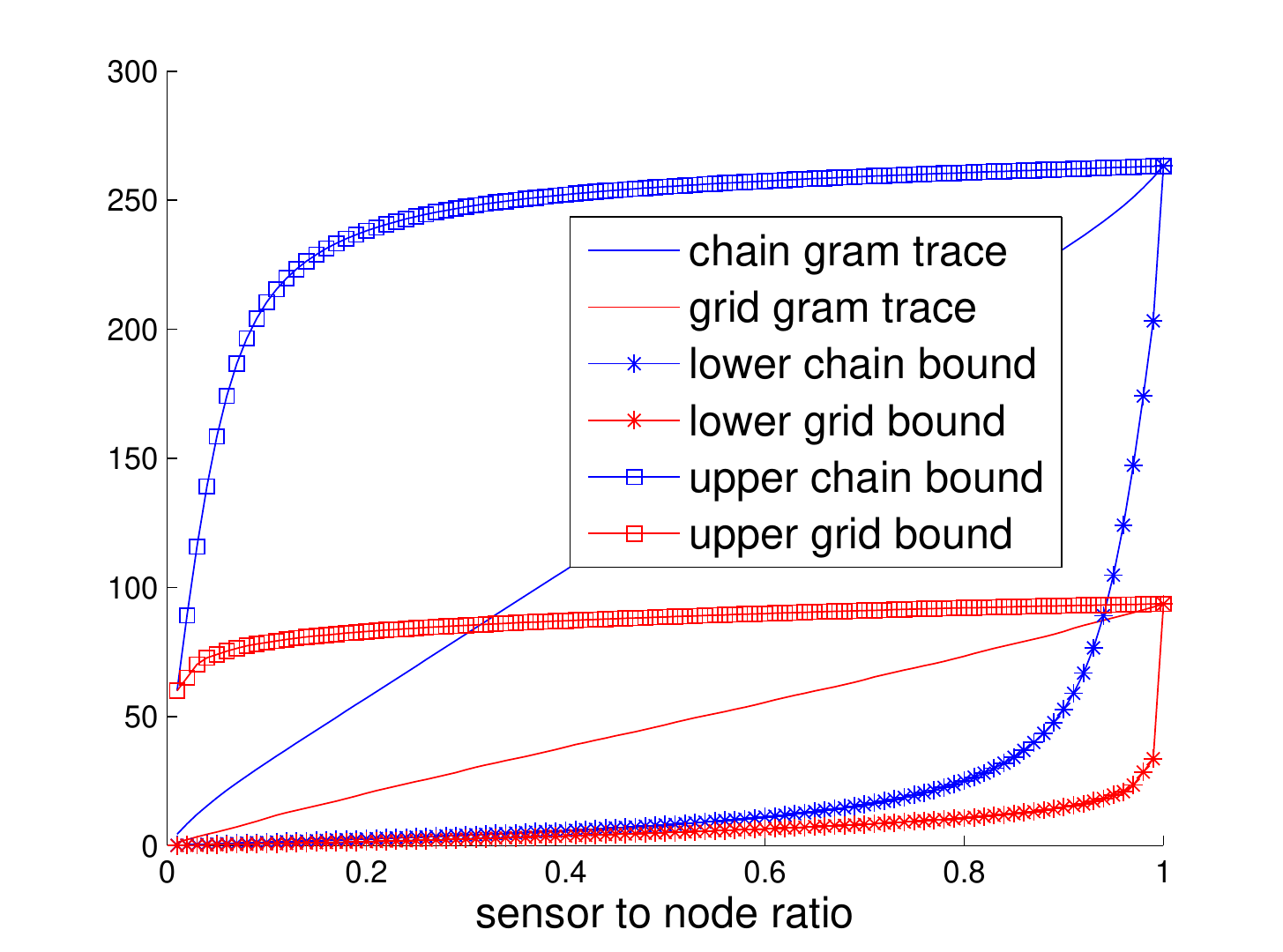}}		
		\subfigure[Networks with $10000$ nodes \label{fig:bound 10000}]{\includegraphics[scale=0.55]{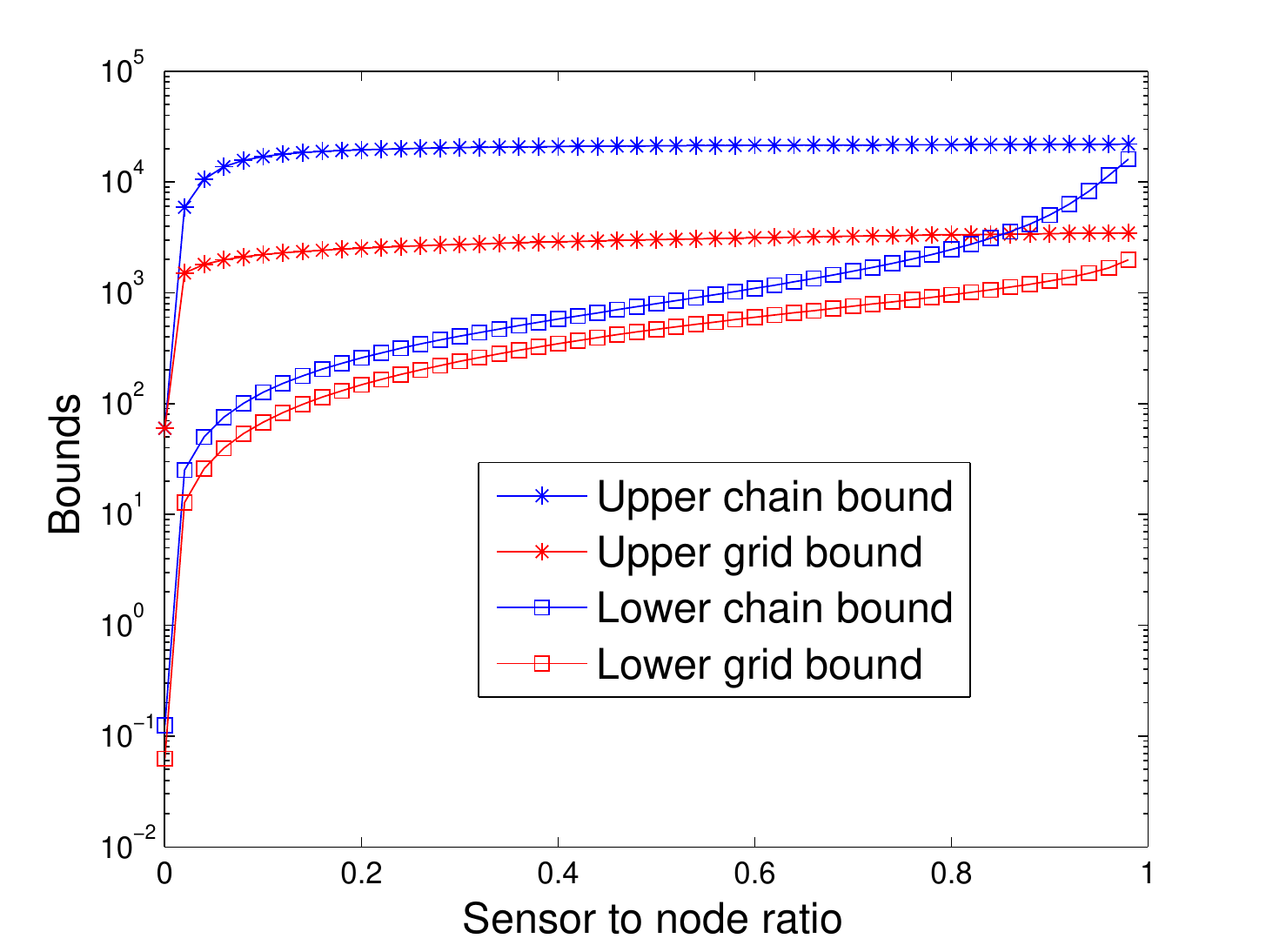}} 	
	\caption{Comparison of the degree of observability based on trace of observability Gramian using its bounds. Trace computation shown in \ref{fig:bound 100} is done numerically using eigenvectors of $ \mathbf{L}(\mathcal{G}) $ }
	\label{fig:Comparison}
\end{figure}




In this section, we analyze the effect of network topology on the accuracy of the field estimation as the number of robots in the network increases. Comparing the results in \ref{fig:estimate chain} and \ref{fig:estimate grid}, it is evident that there is some fundamental limitation arising from the network structure which makes the system with the chain topology practically unobservable.  In the control theory literature, the {\it degree of observability} is used as a metric of a system's observability \cite{Muller:1972:AOC:2244776.2244812}. The {\it observability  Gramian} $\mathit{W}_O(0,T)$ can be used to compute the initial state of an observable linear system from output data over time $t \in [0~T]$ \cite{Hespanha2009}.  This makes it a good candidate for use in quantifying the relative observability among different systems. Due to the duality of controllability and observability, the results associated with one of these properties can be used for the other if interpreted properly.  Commonly used measures of the degree of observability (controllability) are the smallest eigenvalue, the trace, the determinant, and the condition number of the observability (controllability) Gramian \cite{Pasqualetti20147039448,Yan2015,EnyiohaRPJ14}.  For large, sparse networked systems, the Gramian can be highly ill-conditioned, which makes numerical computation of its minimum eigenvalue unstable.  Although researchers have computed bounds on the minimum eigenvalues of the Gramian \cite{PasqualettiZB13}, these bounds did not help to us arrive at a conclusion since they were too close together.




These factors prompted us to use the trace of the observability Gramian as our metric for the degree of observability. Analogous to the interpretation of the controllability Gramian in \cite{PasqualettiZB13}, the trace of the observability Gramian can be interpreted as the average sensing effort applied by a system to estimate its initial state. For a communication network represented by $\mathcal{G}$ with information flow dynamics given by \ref{eqn:sys_dyn}, the trace of the observability Gramian $\mathit{W}_O(0,T)$ is defined as 
\begin{equation}
Trace(\mathit{W}_O(0,T)) = Trace\left ( \int_{0}^{T}\mathbf{e}^{- \mathbf{L}(\mathcal{G})^*t} \mathbf{C}^*\mathbf{C}\mathbf{e}^{- \mathbf{L}(\mathcal{G})t}dt \right ).
\label{eqn:trace_gram}
\end{equation}

Following steps similar to those in \cite{PasqualettiZB13}, we use \ref{theorem:trace_bounds} below to derive upper and lower bounds on the trace of the observability Gramian for networks with chain and grid topologies.  \ref{fig:Comparison} compares these lower and upper bounds for two node populations as a function of the sensor-to-total-node ratio.  It is clear from the plots that the average sensing effort required by the chain network is greater than that of the grid network for a given measurement energy, which is defined as $\left \| \mathbf{Y}(t)  \right \|^2_{L^2([0\ T],\mathbb{R}^k)}$ \cite{PasqualettiZB13}, where $ \mathbf{Y}(t)$ is obtained from \ref{eqn:sys_out}.


\vspace{3mm}

\newtheorem{theorem}{Theorem}
\begin{theorem}
	\label{theorem:trace_bounds}
	Let $\mathcal{G}$ be an unweighted, undirected graph that represents the communication network of a set of $N$ robots with information dynamics and output map given by \ref{eqn:sys_dyn} and \ref{eqn:sys_out}, respectively.  If we label $\mathit{V(\mathcal{G})} $ such that  $k$  observable nodes(robots) in  $\mathit{V(\mathcal{G})} $ where $k \leq N$ are labeled as $1, 2, ..., k$, then $\mathbf{C} = \left[ \mathbf{I}_{k \times k}\ \mathbf{0}_{k \times (N - k)} \right]$. Assuming that $\mathbf{L}(\mathcal{G})$ is diagonalizable and $\lambda_1 \geq \lambda_2 \geq ... \geq \lambda_N $ are its eigenvalues, then there exist real constants $c_1 \leq c_2 \leq ... \leq c_N$ such that  
	\begin{equation}
 \sum_{i=1}^{k} c_i  ~\leq ~Trace\left( \mathit{W}_O(0,T) \right) ~\leq~ \sum_{i=0}^{k-1} c_{n-i},
	\label{eqn:trace_bounds}
	\end{equation}
	where $c_i = \int^{T}_{0} e^{-2\lambda_i t} dt$. \\
\end{theorem}
	
	\begin{proof}
		From the definition of the trace operator, it can be shown that the trace and integral operators are commutative. Using this property and the property that the trace operator is invariant under cyclic permutation \cite{Horn:1985:MA:5509}, \ref{eqn:trace_gram} can be written as
		\begin{align}
 Trace(\mathit{W}_O(0,T)) && \nonumber \\
	&& \hspace{-3cm}	= \int_{0}^{T}Trace\left (\mathbf{e}^{- \mathbf{L}(\mathcal{G})^*t} \mathbf{C}^*\mathbf{C}\mathbf{e}^{- \mathbf{L}(\mathcal{G})t}\right) dt \nonumber\\
	&&	\hspace{-3cm} = \int_{0}^{T}Trace\left ( \mathbf{C}^*\mathbf{C}\mathbf{e}^{\left( - \mathbf{L}(\mathcal{G})t - \mathbf{L}(\mathcal{G})^*t\right)}\right) dt. \nonumber
		\end{align}
	Since the Laplacian of an unweighted, undirected graph is a Hermitian matrix, this equation becomes
		\begin{equation}
		Trace(\mathit{W}_O(0,T))  = \int_{0}^{T}Trace\left ( \mathbf{C}^*\mathbf{C}\mathbf{e}^{- 2 \mathbf{L}(\mathcal{G})t}\right) dt. \nonumber
		\end{equation}
		Let $ \mathbf{L}(\mathcal{G}) = \mathbf{V} \Lambda \mathbf{V}^*$ such that $ \Lambda =Diag(\lambda_1,\lambda_2,...,\lambda_N)$ and the columns of $ \mathbf{V} \in \mathbb{R}^{N \times N} $ are given by the corresponding eigenvectors of   $ \mathbf{L}(\mathcal{G}) $. Then using the decomposition of the matrix exponential \cite{Strang88}, the equation becomes
		\begin{equation}
	Trace(\mathit{W}_O(0,T)) =	\int_{0}^{T}Trace\left ( \mathbf{C}^*\mathbf{C}\mathbf{V}\mathbf{e}^{- 2 \Lambda t}\mathbf{V}^*\right) dt \nonumber
		\end{equation}  
		\begin{equation}
		= ~Trace\left ( \mathbf{C}^*\mathbf{C}\mathbf{V}\left (  \int_{0}^{T}\mathbf{e}^{- 2 \Lambda t}dt\right )\mathbf{V}^*\right).   \nonumber
		\end{equation}  
		The matrix exponential $  \int_{0}^{T}\mathbf{e}^{- 2 \Lambda t}dt $ is a diagonal matrix given by $ Diag\left( \int_{0}^{T}\mathbf{e}^{- 2 \lambda_1 t}dt, \int_{0}^{T}\mathbf{e}^{- 2 \lambda_2 t}dt, ... ,\int_{0}^{T}\mathbf{e}^{- 2 \lambda_N t}dt\right)  $. We define $ c_i = \int_{0}^{T}\mathbf{e}^{- 2 \lambda_i t}dt$.  Then, since $\lambda_1 \geq \lambda_2 \geq ... \geq \lambda_N$, by definition we have that $ c_1 \leq c_2 \leq ... \leq c_N$. 
		
		Let $ \mathbf{M} =  \mathbf{V}\left ( \int_{0}^{T}\mathbf{e}^{- 2 \Lambda t}dt\right )\mathbf{V}^*$. Then we see that $\mathbf{M}$ is a Hermitian matrix with eigenvalues $c_1,c_2,...,c_N$ and the same eigenvectors as $\mathbf{L}(\mathcal{G})$. Also, we find that $ \mathbf{C}^*\mathbf{C} $ is a diagonal matrix with the first $k$ diagonal elements equal to $1$ and the rest equal to 0. Defining $\mathbf{P} = \mathbf{C}^*\mathbf{C} $, we obtain a compact form for the trace of the observability Gramian,
		\begin{equation}
		Trace\left( \mathit{W}_O(0,T) \right) = Trace\left ( \mathbf{P}\mathbf{M}\right).  
		\label{eqn:trace_compact} 
		\end{equation}  
		
		 \ref{eqn:trace_compact} can be reduced to:
		\begin{equation}
		 Trace\left( \mathit{W}_O(0,T) \right) = Trace\left ( \mathbf{P}\mathbf{M}\right) = \sum_{i = 1}^{k} M_{ii}, \nonumber
		\end{equation}  
		where $M_{ii}$ denotes the $i^{th}$ diagonal entry of $ \mathbf{M}$. 
		
		From Theorem 1 of \cite{daboul1990inequalities}, we obtain the following lower bound:
		\begin{equation}
		Trace\left ( \mathit{W}_O(0,T)\right) = \sum_{i = 1}^{k} M_{ii}  \geq \sum_{i=1}^{k} c_i.
		\label{eqn:trace lower bound}
		\end{equation}
		  
		Now by applying \textit{Von Neumann's trace inequality} \cite{Horn:1985:MA:5509} to \ref{eqn:trace_compact} and because $ \mathit{W}_O(0,T) $ is at least positive semidefinite, we find that
		\begin{equation}
		  Trace\left ( \mathbf{P}\mathbf{M}\right) \leq \sum_{i = 0}^{n-1} \sigma \left ( \mathbf{P}\right)_{n-i} \sigma \left ( \mathbf{M}\right)_{n-i} \nonumber
		\end{equation}
		where $\sigma(\cdot)_i$ is the $i^{th}$ singular value of a matrix.  The singular values are arranged in increasing order, $\sigma(\cdot)_1 \leq \sigma(\cdot)_2 \leq ... \leq \sigma(\cdot)_N$, and here they coincide with the eigenvalues of the matrices. Note that only the last $k$ eigenvalues of $\mathbf{P}$ are nonzero and are equal to 1.  Thus, we obtain the upper bound:
		\begin{equation}
	    Trace\left( \mathit{W}_O(0,T) \right) \leq 	\sum_{i=0}^{k-1} c_{n-i}.
	    \label{eqn:trace upper bound}
		\end{equation}	
			
		\end{proof}
		
We can compute these bounds on the trace of the observability Gramian for $\mathbf{L}(\mathcal{G}_c)$ and $\mathbf{L}(\mathcal{G}_g)$ since the eigenvalues of these matrices can be obtained analytically \cite{Edwards2013}.

		
\section{PERFORMANCE ANALYSIS}
\label{sec:Perform analyse}

In this section, we analyze the effect of noise on the output of first-order linear dynamics that evolve on chain and grid network topologies. We assume that the data at each node in the network is affected by white noise with zero mean and unit covariance. Therefore, the augmented system dynamics described by \ref{eqn:sys_dyn} can be written as
	\begin{equation}
	\dot{\mathbf{X}}(t) = - \mathbf{L}(\mathcal{G}) \mathbf{X}(t) + \mathit{W}, \\ 
	 \label{eqn:Aug_sys_dyn}
	\end{equation}
where $\mathit{W} \in \mathbf{R}^N$ denotes a zero mean, unit covariance white noise process.  The output equation is the same as \ref{eqn:sys_out}.

As defined in the robust control literature, the $\mathcal{H}_2$ norm of a system gives the steady-state variance of the output when the input to the system is white noise and when $-\mathbf{L}(\mathcal{G}) $ is Hurwitz \cite{Dullerud2000opac-b1098305}. However, for unstable systems, the finite steady-state variance can be computed only when the unstable modes are unobservable from the outputs \cite{bamjovmitpat12}. For $\mathbf{L}(\mathcal{G})$, zero is the only unsteady mode with corresponding eigenvector $\mathbf{1}_N$, which does not affect the steady-state variance of the output. If we can make the zero mode unobservable, then it is still possible to use the  $\mathcal{H}_2 $ norm as a measure to quantify the effect of noise on the system output. 

In order to do so, we follow the approach in \cite{siami2014graph}, which uses the \textit{first-order Laplacian energy}. This quantity is essentially the $\mathcal{H}_2 $ norm of a system if the matrix $\mathbf{C}$ in \ref{eqn:sys_out} is chosen in such a way that it annihilates the vector $\mathbf{1}_N$. This can be done by defining $\mathbf{C}$ to be an incidence matrix of a graph $\mathcal{G}_k$. Denoting this new $\mathbf{C}$ by $\hat{\mathbf{C}}$, we have that $\mathbf{L}(\mathcal{G}_k)  = \hat{\mathbf{C}}^T\hat{\mathbf{C}}$. Then $\mathbf{L}(\mathcal{G}_k) \mathbf{1}_N = 0$, which implies that $\hat{\mathbf{C}} \mathbf{1}_N = 0$ since $ ker(\hat{\mathbf{C}}) = ker(\hat{\mathbf{C}}^T\hat{\mathbf{C}}) $. Note that $ \hat{\mathbf{C}} $ need not necessarily be the incidence matrix of a graph $ \mathcal{G}_k$; the only condition required is that $\hat{\mathbf{C}}^T\hat{\mathbf{C}} = \mathbf{L}(\mathcal{G}_k)$. 

Now, if $\mathcal{G}_k$ is chosen to be a weighted complete graph $\mathcal{K}_N$ whose edges all have weight $\frac{1}{N}$, then $ \mathbf{L}(\mathcal{G}_k) = \mathbf{I}_{N \times N} - \frac{1}{N} \mathbf{J}_N$.   The first-order Laplacian energy, $\mathcal{H}_{\mathcal{K}_N} ^{(1)}(\mathbf{L}(\mathcal{G}))$, for the corresponding  $ \mathbf{C} $ can be defined from \cite{siami2014graph} as
\begin{equation}
\mathcal{H}_{\mathcal{K}_N} ^{(1)}(\mathbf{L}(\mathcal{G})) = \sum_{i = 1}^{N-1}\frac{1}{2\lambda_i},
\label{eqn:Laplacian energy}
\end{equation}
 where $\lambda_1 \geq \lambda_2 \geq ... \geq \lambda_N = 0 $ are the eigenvalues of  $\mathbf{L}(\mathcal{G})$.


In \ref{fig:performance}, we compare $\mathcal{H}_{\mathcal{K}_N} ^{(1)}(\mathbf{L}(\mathcal{G}))$ for graphs with grid and chain network topologies as a function of the total number of nodes in the network.  The plot shows that the grid network is more effective than the chain network at mitigating the effect of noise on the system output for a given number of nodes.



\begin{figure}[t]
	\centering
		\includegraphics[scale=0.55]{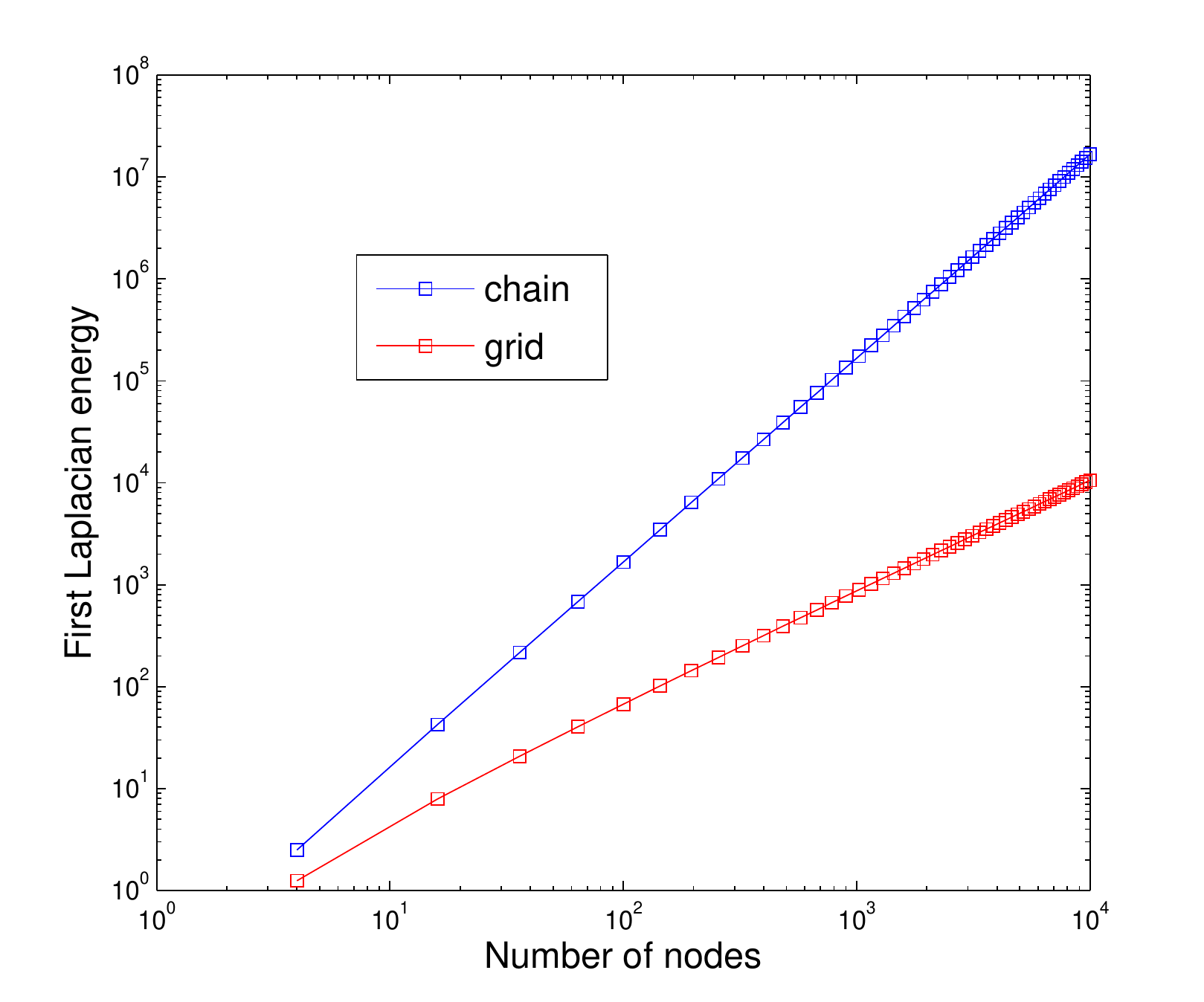}
		\caption{Performance measure based on the first Laplacian energy}
		\label{fig:performance}
	\end{figure}


\section{CONCLUSIONS}
\label{sec:conclusion}

In this work, we have presented a methodology to estimate the initial state of a large networked system of robots with first-order linear information dynamics using output measurements from a subset of robots.  We have quantified the advantages of a grid network over a chain network in the estimation of a two-dimensional scalar field, even though both networks can be made observable by construction. We have also used a performance measure based on the  $\mathcal{H}_2$ norm of the network to characterize the robustness of the network dynamics based on its structure. A straightforward extension of this work is to compare the chain and grid topologies with similar degree distributions using the same methodology. Another interesting aspect to investigate is the effect of structural uncertainty in the networks, which could be done by quantifying the observability radius of the network systems, as defined in \cite{bianchinobservability2016}. In addition, we would like to compare network topologies in an alternative way by viewing $\mathbf{L}(\mathcal{G}_c)$ and $\mathbf{L}(\mathcal{G}_g)$ as approximations to the Laplace operator for 1D and 2D heat equations and analyzing the Gramians of these partial differential equations \cite{van2000selection}. 


\addtolength{\textheight}{-16cm}   



%

\section*{ACKNOWLEDGMENT}
R.K.R. thanks Karthik Elamvazhuthi for his valuable inputs on this work, especially regarding the change of variables in the derivation of the gradient.

\bibliographystyle{IEEEtran} 
\bibliography{IEEEabrv,citation_ref_file}

\end{document}